%% file: main_article.tex
\documentclass[final,hidelinks,onefignum, onetabnum]{siamart251216}

\input{main_shared}

\begin{document}

\maketitle

\begin{abstract}
Given two symmetric positive-definite matrices $A, B \in \mathbb{R}^{n \times n}$, we study the spectral properties of the interpolation $A^{1-x} B^x$ for $0 \leq x \leq 1$. The presence of `common structures' in $A$ and $B$, eigenvectors pointing in a similar direction, can be investigated using this interpolation perspective. Generically, exact log-linearity of the operator norm $\|A^{1-x} B^x\|$ is equivalent to the existence of a shared eigenvector in the original matrices; stability bounds show that approximate log-linearity forces principal singular vectors to align with leading eigenvectors of both matrices. These results give rise to and provide theoretical justification for a multi-manifold learning framework that identifies common and distinct latent structures in multiview data.
\end{abstract}



\section{Introduction and Results}
\label{sec:intro}

\subsection{The problem}
Let $A, B \in \mathbb{R}^{n \times n}$ be two symmetric and positive definite matrices. We assume that $A$ has eigenvalues $\sigma(A) = \left\{\lambda_1, \dots, \lambda_n\right\}$ and $B$ has eigenvalues $\sigma(B) = \left\{ \mu_1, \dots, \mu_n \right\}$. 
We make no additional assumptions on $A$ and $B$ and are motivated by the question whether the eigenvectors of $A$ could, in some natural way, be matched with the eigenvectors of $B$ (the underlying motivation comes from a concrete application discussed in \cref{sec:app_brief}). This question is sufficiently vague that many solutions are possible: for example, one could think of the eigenvectors as two sets of $n$ vectors on $\mathbb{S}^{n-1}$ and then match them by minimizing over some notion of distance. A particular way of matching spectra was proposed in the multimodal manifold learning literature \cite{Katz2025} (see \cref{sec:app_brief} and  \cref{section:app_to_manifold} for details); its effectiveness in concrete applications motivated our interest in the underlying theory.
Since $A$ and $B$ are symmetric and positive definite, their powers $A^{1-z}$ and $B^z$ are well-defined for any $z \in \mathbb{C}$ and, in particular, for real $0 \leq x \leq 1$.  $A^{1-x}$ and $B^x$ are also symmetric and positive definite. Their product $A^{1-x}B^x$ is not necessarily diagonalizable; however, it is a square matrix that has at least $n$ singular values. One could now try to understand the singular values of $A^{1-x} B^x$ for $0 \leq x \leq 1$. 
The main purpose of our paper is to show that the singular values, i.e. the $n$ real-valued functions $x \rightarrow \sigma_k(A^{1-x}B^x)$, for $0 \leq x \leq 1$,
\begin{enumerate}
    \item are an interesting object with an interesting underlying mathematical structure (see \cref{subsec:identifying_common_eig} and  \cref{subsec:stability})
    \item which prove to be useful in specific applications; we discuss the case of multi-manifold learning in \cref{sec:app_brief} and \cref{section:app_to_manifold}.
\end{enumerate}

A very rough motivation is as follows: one sometimes measures the same object in different ways which may end up resulting in two different symmetric positive-definite (kernel) matrices; however, these two matrices should correspond to the same underlying `ground truth', and this similarity should be reflected in their spectrum, where there should be a natural `bijection' between eigenvectors. An example is shown in \cref{Fig:cow_data_projs_3d}: the same overall geometry is captured by similar eigenvectors (used here to color the point clouds).

\begin{figure}[h!] 
    \centering
    \includegraphics[width=0.5\textwidth]{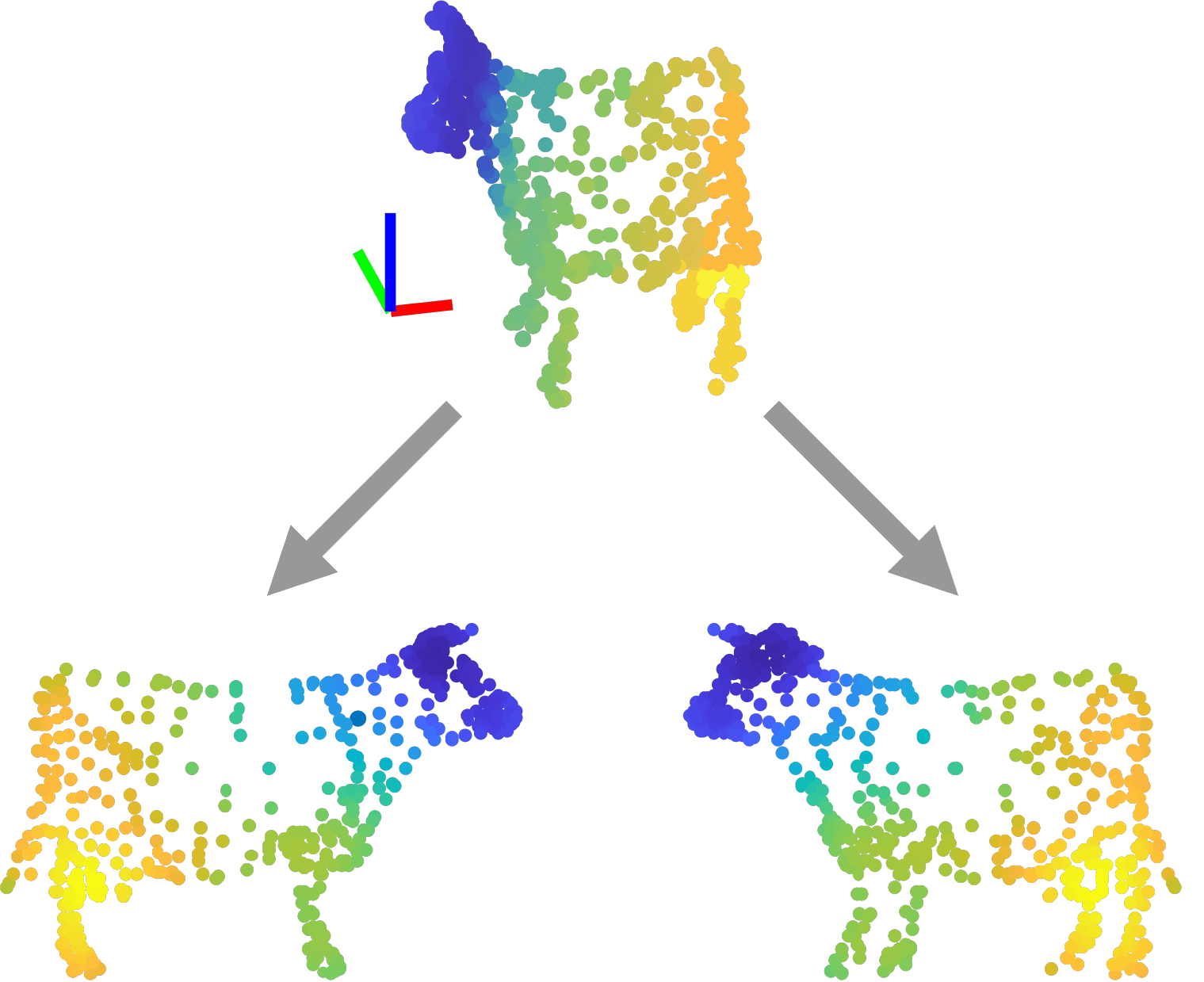}
    \captionsetup{width=\textwidth}
    \caption[Cows illustration: scatter plot]{Two 2D projections of a 3D point cloud of a cow on two different view angles: given only the two two-dimensional point sets (discretized via a kernel into two symmetric positive-definite matrices $A,B \in \mathbb{R}^{n \times n}$), can one automatically discover that the eigenstructure captures the same underlying ground truth? (Details in \cref{sec:app_brief} and \cref{section:app_to_manifold}).
    }
     \label{Fig:cow_data_projs_3d}
\end{figure}


\subsection{Identifying Common Eigenvectors} 
\label{subsec:identifying_common_eig}

We are now able to motivate the first basic result: the largest singular value or, equivalently, the operator norm
$$ \sigma_1(A^{1-x} B^x) = \|A^{1-x} B^x\|$$
has the property that $\|A^{1-x} B^x\| \leq \|A\|^{1-x} \|B\|^x$. Moreover, we will argue that for generic pairs of matrices $A,B$, we have equality if and only if their operator norm is realized by a shared eigenvector $v \in \mathbb{R}^{n}$ normalized to $\|v\|_{\ell^2}=1$ which satisfies
$$ \|A\| = \|A v\| = \|Bv\| = \|B\|.$$
One direction is simple:
\begin{align*}
    \log \left( \|A^{1-x} B^x\| \right) &\leq  \log \left( \|A^{1-x}\| \| B^x\| \right) = \log \left( \|A\|^{1-x} \| B\|^x \right) \\
    &= \log \left( \|A\|^{1-x} \right) + \log \left(  \| B\|^x \right) = (1-x) \log\|A\| + x \log\|B\|.
\end{align*}
If $A$ and $B$ have a common eigenvector $v \in \mathbb{R}^{n}$, say $Av = \lambda v$ and $Bv = \mu v$, then
\begin{align*}
     \|A^{1-x} B^x v\| = \mu^x  \|A^{1-x}  v\| = \lambda^{1-x} \mu^x,
\end{align*}
for which the logarithm is linear. 
One could now wonder about the inverse result: does the linearity of $  \log \left( \|A^{1-x} B^x v\| \right)$ imply that $v$ is an eigenvector of both $A$ and $B$? If, for example, $B=2A$, then $  \log \left( \|A^{1-x} B^x v\| \right)$ is linear for each $v \in \mathbb{R}^n$, we therefore have to ensure that $A$ and $B$ are `different'. Our assumption will be that the ratio of two eigenvalues $\lambda_i/\mu_j$ uniquely identifies $\lambda_i$ and $\mu_j$ (a property satisfied by generic pairs of matrices).

\begin{theorem}[Identifiability]
\label{theorem:log_linear_lines}
Let $A, B \in \mathbb{R}^{n \times n}$ be two symmetric and positive definite matrices. Suppose the map
$ d: \sigma(A) \times \sigma(B) \rightarrow \mathbb{R}_{>0}$ defined by $d(\lambda, \mu) = \lambda/\mu$ is injective and suppose
there exists $0 \neq v \in \mathbb{R}^n$ such that
$$ \log \left\| A^{1 -  x} B^{  x} v\right\| \quad \mbox{is linear,}$$
then $v$ is an eigenvector of both $A$ and $B$.
\end{theorem}

For example, if $A,B \in \mathbb{R}^{n \times n}$ are two diagonal matrices with entries that are chosen uniformly at random from $[1,2]$ that are then sorted in decreasing order. The $n$ functions $\log \sigma_k (A^{1-x} B^{x})$ then form, almost surely, $n$ lines.

\subsection{A stability version}
\label{subsec:stability}

While Theorem 1 is an encouraging fact, it is only applicable if the leading eigenvectors of $A$ and $B$ are \textit{exactly} the same; this is rarely the case in practical applications. Luckily, Theorem 1 remains `morally' true in the case when $A$ and $B$ `almost' share an eigenvector. 
To simplify exposition, we remove the scaling symmetry and assume without loss of generality that $\lambda_1(A) = \|A\| = 1 = \|B\|= \lambda_1(B)$. Then $\|A^{1-x} B^x\| \leq 1$ and Theorem 1 states that subject to the genericity assumption, the case of equality $\|A^{1-x} B^x\| = 1$ for some $0 < x < 1$ implies that $A$ and $B$ have an eigenvector in common.
We can now state the main stability result: if $\|A^{1-x}B^x\|$ is very close to 1, then the left principal singular vector $u \in \mathbb{R}^n$ of $A^{1-x}B^x$ has to point in nearly the same direction as the leading eigenvector $A a_1 = \|A\| a_1 = a_1$. Moreover, the right principal singular vector $v \in \mathbb{R}^n$ of $A^{1-x}B^x$ has to point in nearly the same direction as the leading eigenvector $B b_1 = \|B\| b_1 = b_1$.

\begin{theorem}[Stability]\label{theorem:iff_bound} 
Let $A, B \in \mathbb{R}^{n \times n}$ be symmetric, positive definite, normalized to $\|A\|=1=\|B\|$. We assume $a_1$ and $b_1$, satisfying $Aa_1 = a_1$ and $Bb_1 = b_1$, are $\ell^2$-normalized eigenvectors corresponding to the eigenspace associated with eigenvalue 1, which has multiplicity 1. Let $\lambda_2(A) < 1$ and $\lambda_2(B) < 1$ denote the second largest eigenvalues of $A$ and $B$, respectively. Let $0 \leq x \leq 1$  and let $u \in \mathbb{R}^n$ and $v \in \mathbb{R}^n$ be the principal left and right singular vectors of $A^{1-x}B^{x}$, respectively. Then
\begin{equation}\label{eq:theorem_1_bound_a}
\left|\left\langle u,a_{1}\right\rangle \right|^{2}\geq 1-\frac{1-\|A^{1-x}B^{x}\|^2}{\left(1-x\right)\left(1-\lambda_2(A)\right)},
\end{equation}
and
\begin{equation}
 \label{eq:theorem_1_bound_b}
\left|\left\langle v,b_{1}\right\rangle \right|^{2} \geq 1-\frac{1-\|A^{1-x}B^{x}\|^2}{x  \left(1-\mu_2(B)\right)}.
\end{equation}
\end{theorem}
\textbf{Remarks.} Several remarks are in order.
\begin{enumerate}
    \item \cref{theorem:iff_bound} states that if $\log \|A^{1-x} B^x\|$ is \textit{close} to a line (i.e. $ \|A^{1-x} B^x\|$ is close to 1), then the left principal singular vector of $A^{1-x}B^x$ is close (in inner product) to the eigenvector $a_1$ of $A$ corresponding to the largest eigenvalue of $A$ and the right principal singular vector of $A^{1-x}B^x$ is close to the eigenvector $b_1$ of $B$ corresponding to the largest eigenvalue of $B$.
    \item The factor $(1-x)$ in \eqref{eq:theorem_1_bound_a} and $x$ in \eqref{eq:theorem_1_bound_b} are natural: one cannot hope to get too much information about $A$ from $A^{\varepsilon} B^{1-\varepsilon}$ when $0 < \varepsilon \ll 1$. Likewise, one would not expect $A^{1-\varepsilon} B^{\varepsilon}$ to provide high quality information about $B$ uniformly as $\varepsilon \rightarrow 0^+$.
    \item The factors controlling the spectral gap $1-\lambda_2(A)$ and $1- \mu_2(B)$ are also natural since we are making a pointwise statement about the eigenvectors $a_1, b_1$. If the spectral gap is small, then $\| Aa_2\| = \lambda_2(A) \sim 1$ may almost realize the operator norm. We note that the bound presented has a tighter version using deeper spectral components (see \cref{subsec:better_bound}).
\end{enumerate}

\subsection{Related results}
\label{subsec:app_related_results}

We are not aware of any such results in the literature; however, there are some philosophically related ideas. Our main motivation is a matrix inequality of Alan McIntosh \cite{mcintosh1979heinz}, which generalizes a number of older inequalities: if $A,B \in \mathbb{R}^{n \times n}$ is symmetric and positive-definite and $X \in \mathbb{R}^{n \times n}$ is arbitrary, then for any $0 < r < 1$
$$ \| A^r X B^{1-r} \| \leq \|AX\|^r \|XB\|^{1-r}.$$
This is known to imply the L\"owner-Heinz inequality \cite{lowner1934monotone}, the Heinz-Kato inequality \cite{heinz1951beitrage,kato1952notes}, the Cordes inequality \cite{cordes1987spectral}, and several other such results. The approach of McIntosh is to consider complex interpolation of operators $z \rightarrow A^{z} X B^{1-z}v$ in combination with the maximum principle; it was pointed out by one of the authors \cite{steinerberger2019refined} that such an argument comes, automatically, with stability estimates: for the maximum principle to be sharp, there cannot be too much oscillation; this argument was then carried out in \cite{steinerberger2019refined}. Our arguments follow the same philosophical line of reasoning to obtain a similar structural result in our setting.

Our work, when seen as an application to multi-manifold learning, is closely related to a kernel-based approaches. 
A key method in this direction is alternating diffusion \cite{LEDERMAN2018509, talmon2019latent}: given two matrices $A,B \in \mathbb{R}^{n \times n}$ constructed from different modalities, alternating diffusion considers their (unweighted) product $AB$ to form a composite diffusion operator. It was shown that the leading singular vectors of the product operator are associated with the geometry of the common latent variables.
Several extensions of this idea have been proposed: these include composite diffusion operators \cite{shnitzer2019recovering}, using geodesic interpolation under the affine-invariant Riemannian metric \cite{shnitzer2024spatiotemporal,Katz2025}, compositions of diffusion-type operators across time \cite{froyland2015dynamic, froyland2020dynamic}.
Another related line of research seeks functions that are jointly smooth with respect to multiple kernels \cite{dietrich2022spectral, coifman2023common}.
A classical and conceptually related notion of commonality is provided by canonical correlation analysis (CCA) \cite{hotelling1936relations} and the extension to Kernel CCA \cite{akaho2006kernel, bach2002kernel} and nonparametric CCA \cite{michaeli2016nonparametric}. We are not aware of a fine analysis of complex interpolation of operators having been previously used in the context of multi-manifold learning.

\section{Application to Multi-Manifold Learning}
\label{sec:app_brief}

We briefly describe how the interpolation framework introduced above arises in multimodal manifold learning (more details can be found in Section~\ref{section:app_to_manifold}).
We consider two datasets consisting of \emph{aligned} point clouds
\[
\{s^{(1)}_i\}_{i=1}^n \subset \mathcal{M}_1 \subset \mathbb{R}^{d_1},
\qquad
\{s^{(2)}_i\}_{i=1}^n \subset \mathcal{M}_2 \subset \mathbb{R}^{d_2},
\]
where each pair $(s^{(1)}_i, s^{(2)}_i)$ corresponds to two observations of two manifolds $\mathcal{M}_1$ and $\mathcal{M}_2$ embedded in Euclidean spaces. This setting naturally arises in multimodal data analysis, where different sensing mechanisms capture complementary views of a common phenomenon of interest.
From each point cloud, we construct a symmetric and positive-definite kernel matrix using pairwise affinities via
\[
A_{ij} = \exp\!\left(-\|s^{(1)}_i - s^{(1)}_j\|^2/\varepsilon^{(1)}\right),
\qquad
B_{ij} = \exp\!\left(-\|s^{(2)}_i - s^{(2)}_j\|^2/\varepsilon^{(2)}\right),
\]
followed by standard normalization (see Section~\ref{section:app_to_manifold} for details) resulting in symmetric positive-definite $A, B \in \mathbb{R}^{n \times n}$.

\begin{figure}[h!]
\centering
\makebox[\textwidth][c]{
\subfloat{
\includegraphics[width=0.5\textwidth]{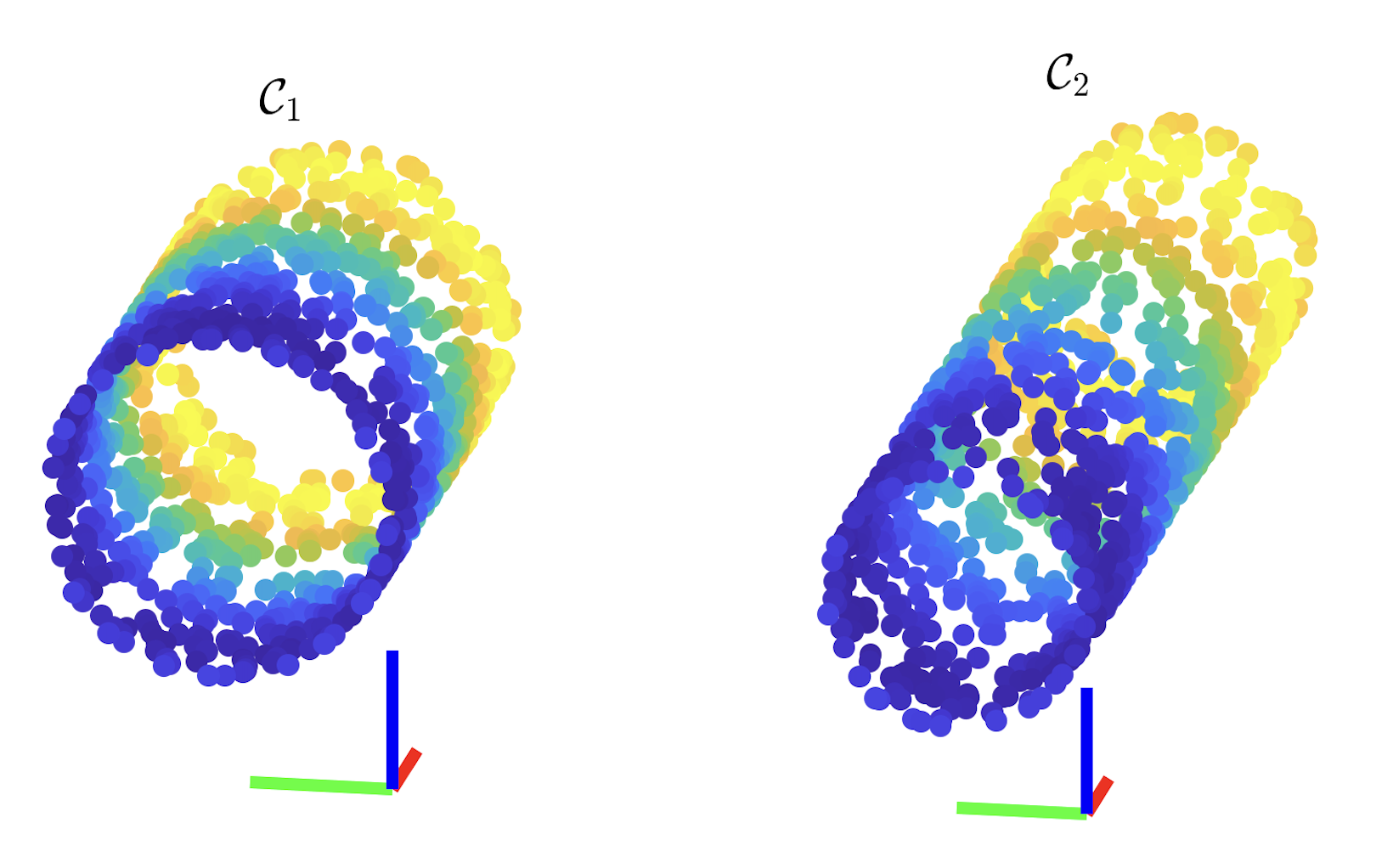}
}
}
\captionsetup{width=\textwidth}
\caption[Cylinder illustration: scatter plot]{Two aligned point clouds sampled from two 2D cylinders embedded in $\mathbb{R}^3$. The height (red) axis is shared among the point clouds, whereas the azimuthal angle is distinct.}
\label{Fig:cylinders_scatter}
\end{figure}
The central object of interest is the interpolated family
\[
\gamma(x) = A^{1-x} B^x, \qquad x \in [0,1],
\]
whose spectral properties encode relationships between the two point clouds.
To analyze $\gamma(x)$, we consider the singular values of $A^{1-x}B^x$ as functions of $x$. This leads to the construction of a \emph{singular value flow diagram} (SVFD), which tracks the evolution of the leading singular values along the interpolation path.

\begin{figure}[h!]
\centering
\makebox[\textwidth][c]{
\includegraphics[width=0.8\textwidth]{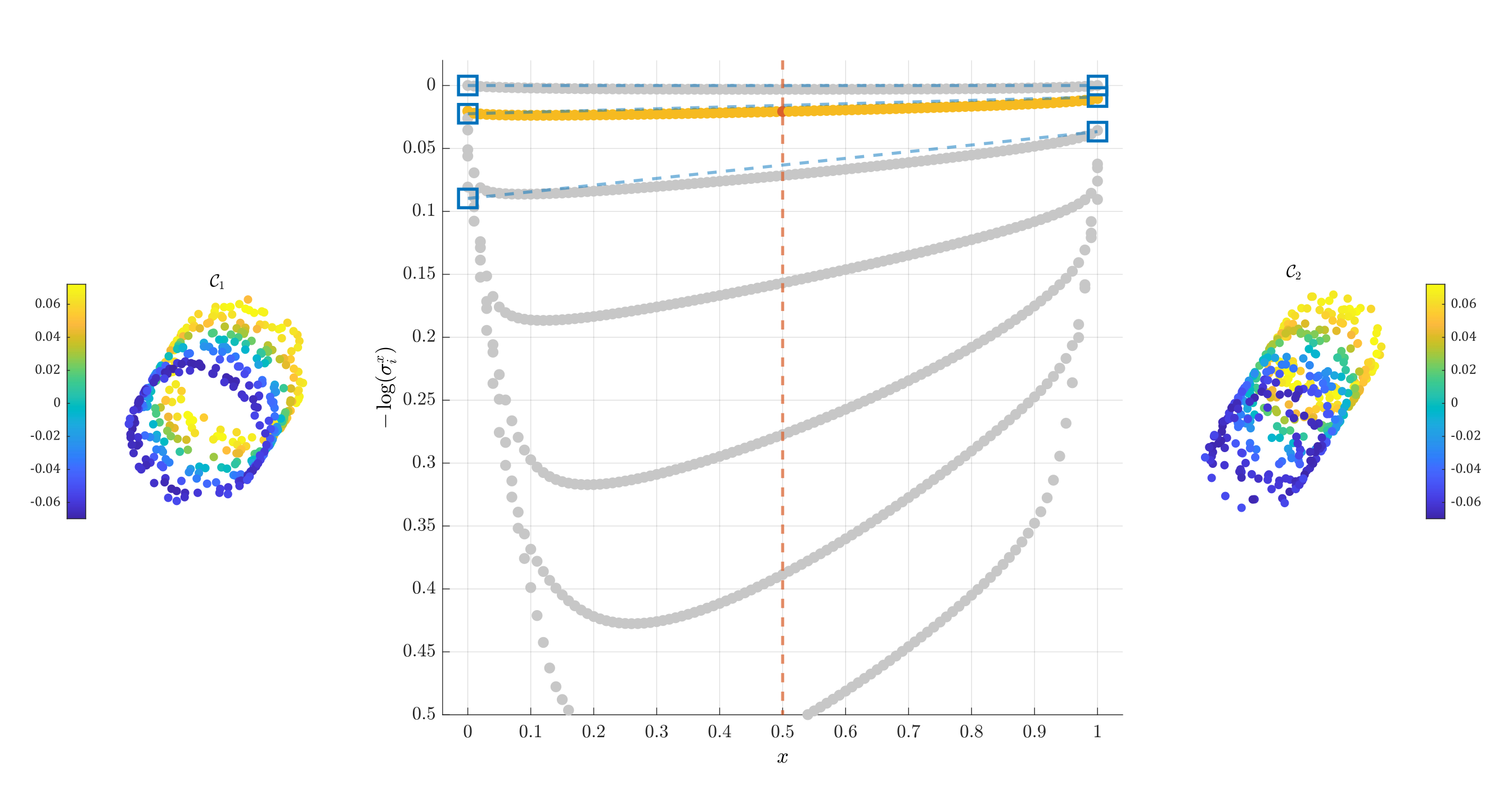}
}
\captionsetup{width=\textwidth}
\caption[Cylinders illustration: Common SVFD]{
The SVFD of the point clouds depicted in \cref{Fig:cylinders_scatter}. 
The empirical singular values of the interpolated kernels connecting the second common eigenvalues (shown left and right) are highlighted in yellow.} 
\label{Fig:cylinders_diagram_vec_2_2}
\end{figure}

In practice, this is done by sampling a discrete set of points $x_k \in [0,1]$, computing the leading singular values of $A^{1-x_k}B^{x_k}$ at each point, and plotting their logarithms as functions of $x_k$. The resulting curves provide a compact representation of how spectral components evolve between the two matrices.
Our theoretical results in \cref{sec:intro} provide a rigorous interpretation of these diagrams: approximately log-linear trajectories correspond to spectral components shared between $A$ and $B$, while curved trajectories indicate distinct components.
We illustrate the approach on a synthetic example consisting of two cylindrical manifolds with a shared latent variable. The construction is described in detail in \cref{section:app_to_manifold}.

\cref{Fig:cylinders_scatter} shows the sampled point clouds, where the vertical coordinate represents a common latent variable, while the angular coordinates differ between the two datasets. The resulting SVFD is shown in \cref{Fig:cylinders_diagram_vec_2_2}. In this figure, several singular value trajectories exhibit near log-linear behavior across the interpolation parameter $x$, indicating spectral components that are shared between the two point clouds. In particular, the highlighted trajectory (shown in yellow) closely follows a straight line in the logarithmic scale, consistent with the theoretical characterization of common eigenvectors.
The insets in \cref{Fig:cylinders_diagram_vec_2_2} further illustrate this phenomenon by coloring the two cylindrical point clouds according to the corresponding left and right singular vectors at an intermediate interpolation point (here $x=0.5$). The coloring reveals a coherent structure across both point clouds, with the variation aligned along the vertical axis, confirming that this spectral component captures the common latent variable.
\begin{figure}[h!]
\centering
\makebox[\textwidth][c]{
\includegraphics[width=0.8\textwidth]{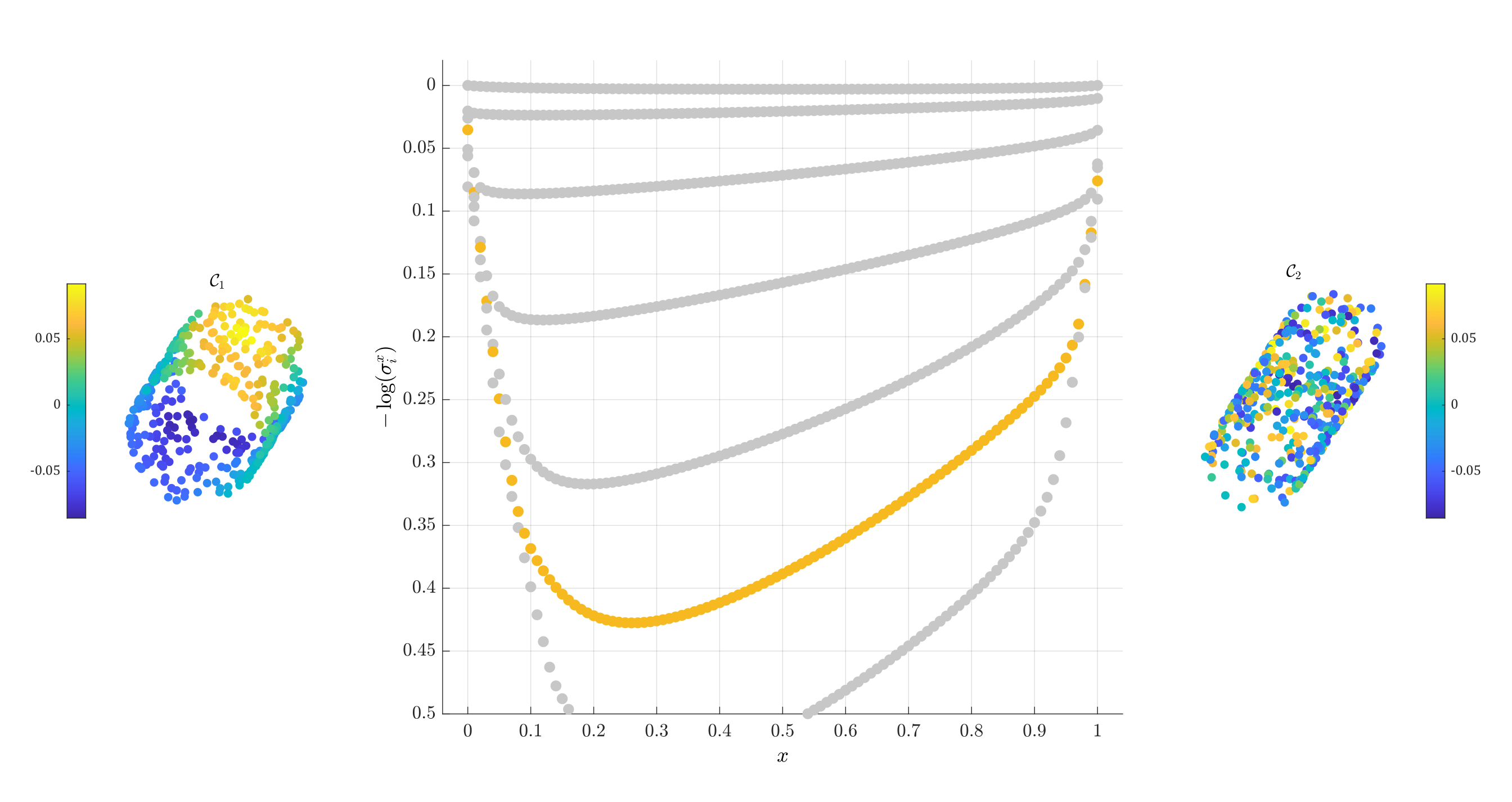}
}
\captionsetup{width=\textwidth}
\caption[Cylinders illustration: Distinct SVFD]{
The same as \cref{Fig:cylinders_diagram_vec_2_2}, but highlighting the empirical singular values originating at $x=0$ from the fourth largest eigenvalue of $\mathcal{C}_1$. The point clouds on the left and right are colored by the corresponding eigenvector. The curve is far from a line, the eigenvectors have little in common.} 
\label{Fig:distinct_cylinders_diagram_vec_4}
\end{figure}
In contrast, in \cref{Fig:distinct_cylinders_diagram_vec_4}, we highlight a trajectory associated with a distinct component. In the SVFD, this trajectory deviates significantly from log-linearity, exhibiting pronounced curvature. This behavior reflects the lack of a shared eigenstructure between the corresponding components of $A$ and $B$.
The insets in \cref{Fig:distinct_cylinders_diagram_vec_4} show the cylinders colored using the singular vector associated with this trajectory. Unlike the previous case, the coloring patterns are not consistent across the two point clouds: a structured harmonic pattern visible on one cylinder does not transfer coherently to the other. This lack of geometric alignment indicates that the corresponding spectral component does not represent a shared structure.
This example demonstrates that the geometry of the singular value trajectories provides a direct and interpretable signature of common versus distinct spectral components.

\section{Proofs}

\subsection{Proof of Theorem \ref{theorem:log_linear_lines}} 
\label{subsec:spectral_stability_proof}
\begin{proof}[Proof of Theorem \ref{theorem:log_linear_lines}]  Let us assume that
$$ \log \left\langle A^{1 -  x} B^{  x} v,  A^{1-  x}  B^{ x} v  \right\rangle \quad \mbox{is linear.}$$
This means that there exist $c_1, c_2 \in \mathbb{R}$ such that
$ \left\langle A^{1 -  x} B^{  x} v,  A^{1-  x}  B^{ x} v  \right\rangle = c_1 e^{c_2 x}.$
We first start by simplifying the expression: assuming that 
$$ Av = \sum_{k=1}^{n} \lambda_k \left\langle v, a_k \right\rangle a_k \qquad \mbox{and} \qquad Bv = \sum_{k=1}^{n} \mu_k \left\langle v, b_k \right\rangle b_k,$$ 
we have
$$ B^{x}v = \sum_{l = 1}^{n}{  \mu_l^{x}  \left\langle v, b_l \right\rangle b_l} $$
and furthermore 
$$ A^{1 -  x}  B^{x}v = \sum_{k=1}^{n}{ \lambda_k^{1 - x} \left\langle  B^{x}v , a_k \right\rangle a_k}$$
and therefore
\begin{align*}
\left\langle A^{1 -  x} B^{  x} v,  A^{1-  x}  B^{ x} v  \right\rangle  &= \left\langle \sum_{k=1}^{n}{ \lambda_k^{1-x} \left\langle  B^{x}v , a_k \right\rangle a_k},  \sum_{k=1}^{n}{ \lambda_k^{1-x} \left\langle  B^{x}v , a_k \right\rangle a_k} \right\rangle_{} \\
&= \sum_{k=1}^{n}{ \lambda_k^{2-2x} \left\langle  B^{x}v , a_k \right\rangle_{}^2}.
\end{align*}
Altogether, this implies
$$ \left\langle A^{1 -  x} B^{  x} v,  A^{1-  x}  B^{ x} v  \right\rangle  =  \sum_{k=1}^{n} \lambda_k^{2-2x} \left(  \sum_{l=1}^{n} \mu_l^x \left\langle v, b_l \right\rangle \left\langle b_l, a_k \right\rangle \right)^2 = c_1 e^{c_2 x}.$$
For the remainder of the argument, we will exploit the algebraic structure:
writing
$$ \alpha_k = \log{(\lambda_k)},~\beta_l = \log{(\mu_l)}, \quad \mbox{and} \quad c_{k,l} = \lambda_k  \left\langle v, b_l \right\rangle   \left\langle  b_l, a_k \right\rangle$$
allows to notationally simplify the equation to
$$ \sum_{k=1}^{n}{e^{- 2 \alpha_k  x} \left( \sum_{l=1}^{n}{ c_{k,l} e^{\beta_l  x } } \right)^2} = c_1 e^{c_2 x}.$$
We note that $v \neq 0$ and both $A, B$ are positive definite, their eigenvectors form a basis of $\mathbb{R}^n$ and therefore, there exists at least one $c_{k,l} \neq 0$. 
Let us now first assume that $A$ and $B$ are both simple: their eigenvalues have multiplicity 1. 
We then define the quantities
$$
 \underline{\sigma} = \min \left\{ 2\beta_l - 2\alpha_k: c_{k,l} \neq 0 \right\} \qquad \mbox{and} \qquad
 \overline{\sigma} = \max \left\{ 2\beta_l - 2\alpha_k: c_{k,l} \neq 0 \right\}.
$$
These numbers give the smallest and largest occurring frequencies: since the spectrum is assumed to be simple, for each $k$ there exists at most one $l$ such that $2 \beta_{l} - 2\alpha_k$ is maximized or minimized. Therefore,
\begin{align*}
 \sum_{k=1}^{n}{ e^{- 2 \alpha_k x} \left( \sum_{l=1}^{n}{ c_{k,l} e^{\beta_l x } } \right)^2}  &=  
 e^{\underline{\sigma} x} \left(\sum_{k,l = 1 \atop 2\beta_l - 2\alpha_k = \underline{\sigma} }^{n}{c_{l,k}^2} \right) +  \sum_{j}{d_j e^{\gamma_j x}} \\
 &+ e^{\overline{\sigma} x} \left(\sum_{k, l = 1 \atop 2\beta_l - 2\alpha_k = \overline{\sigma}}^{n}{c_{l,k}^2} \right),
\end{align*}
where the $d_j, \gamma_j$ could be explicitly computed and the arising frequencies satisfy $\underline{\sigma} < \gamma_j < \overline{ \sigma}$. Note that the cross terms of the form $\beta_l + \beta_m - 2\alpha_k$ for $l \neq m$ do not affect the extreme frequencies, and therefore, absorbed into the intermediate terms $d_j e^{\gamma_j x}$. In addition, by construction,
$$ \sum_{k,l = 1 \atop 2\beta_l - 2\alpha_k = \underline{\sigma} }^{n}{c_{l,k}^2}  \neq 0 \neq \sum_{k, l = 1 \atop 2\beta_l - 2\alpha_k = \overline{\sigma}}^{n}{c_{l,k}^2}.$$
However, in order for this expression to be $c_1 e^{c_2 x}$, we have to have
$$ \underline{\sigma} = c_2 = \overline{\sigma}.$$
This implies that whenever $c_{k,l} \neq 0$, then $2 \beta_l - 2\alpha_k = c_2$.  This equation, in turn, has a unique solution (due to the assumption of an injective $d(\mu,\lambda) = \frac{\lambda}{\mu}$) from which we deduce that there exists a single pair $(k,l)$ such that $c_{k,l} \neq 0$.  This means that there exists a single pair $(k,l)$ for which
$$   \left\langle v, b_l \right\rangle   \left\langle  b_l, a_k \right\rangle \neq 0.$$
We note that for each $l$ there exists at least one $k$ for which $\left\langle  b_l, a_k \right\rangle \neq 0$. This means there exists exactly one $l$ such that
$  \left\langle v, b_l \right\rangle \neq 0$ which implies that $v = b_l \|v\|$ and therefore $v$ is an eigenvector of $B$. Then, however, fixing this value of $l$, there can exist at most one $k$ such that $\left\langle  b_l, a_k \right\rangle \neq 0$, which implies that $v$ is also an eigenvector of $A$. 
It remains to deal with the general case. If the eigenvalues of both matrices can have multiplicities, then, arguing in exactly the same way as above, we see that we can write
$$
 \sum_{k=1}^{n}{ e^{- 2 \alpha_k x} \left( \sum_{l=1}^{n}{ c_{k,l} e^{\beta_l x } } \right)^2}  =  
 e^{\underline{\sigma} x} A_1 +  \sum_{j}{d_j e^{\gamma_j x}} + e^{ \overline{\sigma} x} A_2$$
where the expressions for $A_1$ and $A_2$ are now slightly more involved.  Since $2 \beta_l - 2 \alpha_k = \underline{\sigma}$ has a unique solution, we can call the corresponding eigenvalues $\alpha$ and $\beta$ (keeping in mind that they might have a nontrivial multiplicity).
 A short computation shows that
\begin{align*}
 A_1 &= \sum_{\alpha_k = \alpha} \left( \sum_{\beta_l = \beta} c_{k,l} \right)^2 = \sum_{\alpha_k = \alpha} \left( \sum_{\beta_l = \beta} \lambda_k  \left\langle v, b_l \right\rangle   \left\langle  b_l, a_k \right\rangle \right)^2 \\
 &= e^{2\alpha} \sum_{\alpha_k = \alpha} \left( \sum_{\beta_l = \beta} \left\langle v, b_l \right\rangle   \left\langle  b_l, a_k \right\rangle \right)^2.
 \end{align*}
 We recall an elementary fact for Hilbert spaces: if $\left\{x_1, \dots, x_n\right\}$ is an orthonormal basis of a subspace $S$ of some Hilbert space, then for all 
 $ x\in S$ and all $y \in H$,
 $$ \left\langle x, y \right\rangle = \left\langle \sum_{i=1}^{n} \left\langle x, x_i\right\rangle x_i, y \right\rangle =\sum_{i=1}^{n} \left\langle x, x_i \right\rangle \left\langle x_i, y \right\rangle.$$
 Therefore, using $\pi_{\beta}:\mathbb{R}^n \rightarrow \mathbb{R}^n$ to denote the orthogonal projection onto the eigenspace corresponding to eigenvalue $e^{\beta}$, we have
 $$ A_1 =   e^{2\alpha} \sum_{\alpha_k = \alpha} \left\langle \pi_{\beta} v, a_k \right\rangle^2.$$
 Using the Pythagorean theorem and using $\pi_{\alpha}:\mathbb{R}^n \rightarrow \mathbb{R}^n$ to denote the orthogonal projection onto the eigenspace corresponding to eigenvalue $e^{\alpha}$, we have
  $$ A_1 =   e^{2\alpha} \sum_{\alpha_k = \alpha} \left\langle \pi_{\beta} v, a_k \right\rangle^2 = e^{2\alpha} \left\| \pi_{\alpha} \pi_{\beta} v\right\|^2.$$
We note that if, for any eigenvalue $\beta$, the vector $\pi_{\beta} v \neq 0$, then there exists at least one other eigenvalue $\alpha$ for which $\pi_{\alpha} \pi_{\beta} \neq 0$. This means that $v$ has to be an eigenvector of $B$. Simultaneously, for any vector $v \neq 0$ there exists at least one eigenvalue $\alpha$ such that $\pi_{\alpha} v \neq 0$. This proves the desired statement.
\end{proof}

\subsection{Proof of Theorem \ref{theorem:iff_bound}}
 \subsubsection{Preliminaries}
 We first recall that if $A \in \mathbb{R}^{n \times n}$ is a symmetric and
positive definite matrix with eigenvalues and eigenvectors given by $A a_k = \lambda_k a_k$, then the complex power
$A^z$ for $z \in \mathbb{C}$ is defined by 
$$ A^{z}v=\sum_{k=1}^{n}\lambda_{k}^{z}\left\langle v,a_{k}\right\rangle a_{k}.$$
The purpose of this section is to recall some basic facts regarding complex powers of linear operators.
\begin{lemma}
If $A \in \mathbb{R}^{n \times n}$ is a symmetric and positive definite matrix, then $A^{it}$ is unitary for all $t \in \mathbb{R}$.
\end{lemma}
\begin{proof}
Note that, for $\lambda\geq0$, 
\begin{equation*}
\lambda^{it}=\cos\left(\left(\log\lambda\right)t\right)+i\sin\left(\left(\log\lambda\right)t\right)=e^{i\left(\log\lambda\right)t}.
\end{equation*}
The result then follows by explicit computation since
\begin{align*}
\left\Vert \sum_{k=1}^{n}\lambda_{k}^{it}\left\langle v,a_{k}\right\rangle a_{k}\right\Vert ^{2} & =\left\Vert \Re\sum_{k=1}^{n}\lambda_{k}^{it}\left\langle v,a_{k}\right\rangle a_{k}+\Im\sum_{k=1}^{n}\lambda_{k}^{it}\left\langle v,a_{k}\right\rangle a_{k}\right\Vert ^{2}\\
 & =\left\Vert \sum_{k=1}^{n}\cos\left(\log\left(\lambda_{k}\right)t\right)\left\langle v,a_{k}\right\rangle a_{k}\right\Vert ^{2}\\
 &+\left\Vert \sum_{k=1}^{n}\sin\left(\left(\log\lambda_{k}\right)t\right)\left\langle v,a_{k}\right\rangle a_{k}\right\Vert ^{2}\\
 & =\sum_{k=1}^{n}\left[\cos^{2}\left(\left(\log\lambda_{k}\right)t\right)+\sin^{2}\left(\left(\log\lambda_{k}\right)t\right)\right]\left\vert\left\langle v,a_{k}\right\rangle\right\vert^{2} \\
 &=\sum_{k=1}^{n}\left\vert\left\langle v,a_{k}\right\rangle\right\vert^{2}=\|v\|^{2}.
\end{align*}
\end{proof}

\begin{lemma}\label{appendix:cauchy_schwarz}
If $A,B \in \mathbb{R}^{n \times n}$ are two symmetric and positive definite matrices normalized to $\|A\| = 1 = \|B\|$.  Then, for all $0 \leq \Re(z) \leq 1$, we have
$$  \| A^{1-z} B^z\| \leq 1.$$
\end{lemma}
\begin{proof}
The Cauchy-Schwarz inequality is still valid in the holomorphic case since
\begin{equation*}
\left|\left\langle v,w\right\rangle _{\mathbb{R}}\right|=\left|\sum_{i=1}^{n}v_{i}w_{i}\right|\leq\sum_{i=1}^{n}\left|v_{i}w_{i}\right|\leq\left(\sum_{i=1}^{n}\left|v_{i}\right|^{2}\right)^{\frac{1}{2}}\left(\sum_{i=1}^{n}\left|w_{i}\right|^{2}\right)^{\frac{1}{2}}=\|v\|\|w\|.
\end{equation*}
\par Now, consider $z=0+it$ for $t\in\mathbb{R}$. Using
Cauchy-Schwarz, we have that 
\begin{equation*}
\left|\left\langle A^{1-it}B^{it}v,A^{1-it}B^{it}v\right\rangle _{\mathbb{R}}\right|\leq\left\Vert A^{1-it}B^{it}v\right\Vert ^{2}.
\end{equation*}
$A^{-it}$ and $B^{it}$ are unitary matrices and $\left\Vert v\right\Vert =1$,
therefore 
\begin{equation*}
\left\Vert A^{1-it}B^{it}v\right\Vert =\left\Vert AB^{it}v\right\Vert \leq\left\Vert A\right\Vert \left\Vert B^{it}v\right\Vert =1.
\end{equation*}
By the same reasoning, we have that for $z=1+it$ 
\begin{equation*}
\left\Vert A^{-it}B^{1+it}v\right\Vert =\left\Vert B^{1-it}v\right\Vert \leq\left\Vert B\right\Vert \left\Vert B^{-it}v\right\Vert =1.
\end{equation*}
Using the trivial estimate 
\begin{align*}
\|A^{1-z}B^{z}v\| & \leq\|A^{1-z}\|\|B^{z}\|\leq\|A\|^{1-\Re(z)}\|B\|^{\Re(z)}\\
 & \leq\max(\|A\|,1)\max(\|B\|,1)
\end{align*}
Under the assumption $\|A\|=\|B\|=1$ we obtain the desired bound.
\end{proof}
We conclude with a short Lemma for a harmonic function defined on the strip
$ \mathcal{D} = \left\{z \in \mathbb{C}\colon 0 \leq \Re(z) \leq 1 \right\}.$
\begin{lemma}
Let $F: D \rightarrow \mathbb{R}$ be a harmonic function satisfying $\|F\|_{L^{\infty}(D)} \leq 1$. If, for some $0 \leq x \leq 1$, we have
$F(x,0) \geq 1 - \varepsilon$, then we have
$$ \max_{y \in \mathbb{R}} F(0,y) \geq 1 - \frac{\varepsilon}{1-x} \quad \mbox{and} \quad \max_{y \in \mathbb{R}} F(1,y) \geq 1 - \frac{\varepsilon}{x}.$$
\end{lemma}
\begin{proof}
Let $(Z_t)_{t \geq 0}$ be a standard two-dimensional Brownian motion starting at $Z_0 = (x,0)$, and let $\tau = \inf \{t > 0 : Z_t \notin D\}$ be the first exit time from the domain. $Z_t$ is an It\^o process by definition.
 Since $F$ is harmonic it is a twice continuously differentiable function, so by It\^o's formula \cite{oksendal_ito_2003} $F(Z_t)$ is also an It\^o process, whose evolution is given by:
$$ dF(Z_t) = \frac{\partial F}{\partial  t}(Z_t)dt+ \nabla F(Z_t) \cdot dZ_t + \frac{1}{2} \Delta F(Z_t) \, d^2Z_t. $$
Since $F$ is harmonic and time invariant, $\Delta F = \frac{\partial F}{\partial  t}= 0,$ , the drift term $dt$ as well as the second order term vanish. Consequently, the process $M_t = F(Z_t)$ is a local martingale (a drift-less process).
Furthermore, since $F$ is bounded ($\|F\|_{L^{\infty}} \leq 1$) and the Brownian motion exits the strip $\mathcal{D}$ almost surely, the conditions for the Optional Stopping Theorem are satisfied. This allows us to equate the function's value at the starting point to its expected value at the exit time
$ F(x,0) = \mathbb{E}[M_0] = \mathbb{E}[M_{\tau}] = \mathbb{E}\left[ F(Z_{\tau}) \right]. $
The boundary $\partial D$ consists of the left line $L = \{0\} \times \mathbb{R}$ and the right line $R = \{1\} \times \mathbb{R}$. The probability of the Brownian motion exiting through the right boundary corresponds to the initial location along the $x$ axis:
$$ \mathbb{P}(Z_{\tau} \in L) = 1-x, \quad \text{and} \quad \mathbb{P}(Z_{\tau} \in R) = x. $$
We decompose the expectation over these two exit events:
$$ F(x,0) = (1-x) \cdot \mathbb{E}[F(Z_{\tau}) \mid Z_{\tau} \in L] + x \cdot \mathbb{E}[F(Z_{\tau}) \mid Z_{\tau} \in R]. $$
Using the assumption $F(x,0) \geq 1-\varepsilon$ and the global bound $\sup_{z \in \partial D} F(z) \leq 1$, 
\begin{align*}
1 - \varepsilon &\leq (1-x) \sup_{y \in \mathbb{R}} F(0,y) + x \sup_{y \in \mathbb{R}} F(1,y) \\
&\leq (1-x) \sup_{y \in \mathbb{R}} F(0,y) + x \cdot 1.
\end{align*}
Rearranging the inequality yields:
$$ (1-x) \sup_{y \in \mathbb{R}} F(0,y) \geq 1 - x - \varepsilon, $$
which implies
$$ \sup_{y \in \mathbb{R}} F(0,y) \geq 1 - \frac{\varepsilon}{1-x}. $$
The bound for the right boundary follows by symmetry.
\end{proof}

\subsubsection{Sketch of the proof}
We follow a similar approach as in \cite{steinerberger2019refined}. We will be working on the fundamental strip
\begin{equation}\label{eq:strip_domain}
\mathcal{D} = \left\{ x+iy\in\CC : x\in\left[0,1\right], y\in \RR \right\}.
\end{equation}
\begin{figure}[h!]
    \centering
\begin{tikzpicture}
   \draw[->] (-3, 0) -- (3, 0) node[right] {$x$};
  \draw[->] (-2, -1.5) -- (-2, 1.5) node[above] {$y$};
    \draw[-, dashed, red] (-1.96, -0.25) -- (-1.96, 1.5);
    \draw[-, dashed, red] (-1.96, -1.5) -- (-1.96, -0.5) node[left, yshift=-2] {$x=0$};
    \draw[-, dashed, blue] (1, -0.25) -- (1, 2);
    \draw[-, dashed, blue] (1, -1.5) -- (1, -0.5) node[right, yshift=-2] {$x=1$};
    \node [rectangle,
    text = black,
    fill = {rgb:red,1;green,1;blue,1}, fill opacity=0.1,  draw opacity=1, text opacity=1,
    minimum width=3cm,minimum height=3.6cm] (r) at (-0.48,-0.05) {$ $};
\node at (-0.48,0.25) {$\mathcal{D}$};
\node[yshift=-2] at (r.south) {$\vdots$};
\node[yshift=8] at (r.north) {$\vdots$}; {$\vdots$};
\end{tikzpicture}
    \caption{A sketch of the domain $\mathcal{D}$.}
    \label{Fig:domain}
\end{figure}
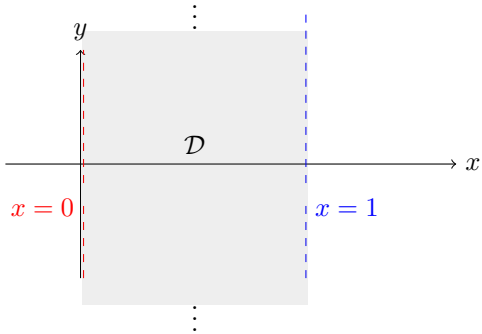

Instead of analyzing the norm of the interpolated operator directly, we will, for any arbitrary $v \in \mathbb{R}^n$, study the behavior of the expression
$$
 x \rightarrow    \left\langle A^{1-x}B^{x}v,A^{1-x}B^{x}v \right\rangle
$$
as a function of $0 \leq x \leq 1$. We note that if $v$ happens to be the principal singular vector of $A^{1-x} B^x$, then this expression is merely the square of the operator norm. Such quantities are often easier to analyze in the complex plane, and we will generalize the interpolation scheme from the real interval to $\mathcal{D}$ to the fundamental strip by instead considering the functions 
$
z \rightarrow\left\langle A^{1-z}B^{z}v,A^{1-z}B^{z}v\right\rangle _{\mathbb{R}} 
$
as well as the dual object
$
 z \rightarrow\left\langle u^*A^{1-z}B^{z},u^*A^{1-z}B^{z}\right\rangle _{\mathbb{R}}
$
both of which are holomorphic by defining $\left\langle v,w\right\rangle _{\mathbb{R}} \triangleq \sum_{k=1}^{n}v_{i}w_{i}$. We note that this reduces to the earlier expression whenever $z = t + 0i$ for $0 \leq t \leq 1$. We can now make use of the fact that any holomorphic function inside a domain is uniquely defined by its values on the boundary. Moreover, since the fundamental strip $\mathcal{D}$ is geometrically rather simple, this is completely explicit (see e.g.  \cite{widder1961strip}).
Every analytic complex-valued function $f:D\mapsto\CC$ can be represented as follows
\begin{align*}\label{eq:integral_holo_func_rep}
f\left(x+iy\right)
&= \frac{1}{2\pi}\int_{-\infty}^{\infty}{P\left(x,t-y\right)\cdot f\left(0+it\right)dt} \\
&+ \frac{1}{2\pi}\int_{-\infty}^{\infty}{P\left(1-x,t-y\right)\cdot f\left(1+it\right)dt},
\end{align*}
where $P$ is a Poisson kernel given by
\begin{equation}\label{eq:pois_kernel}
P\left(x,y\right)=\frac{\pi\sin\left(\pi x\right)}{\cosh\left(\pi y\right)-\cos\left(\pi x\right)}.
\end{equation}
This allows us to reduce the analysis of the special case $A^{1-z}B^z$ for $0 \leq \Re(z) \leq 1$ to that of $A^{1-it} B^{it}$ as well as $A^{-it} B^{1+it}$.

\subsubsection{Proof}
\label{subsec:proof_of_theorem_two}
\begin{proof}

We fix $0 \leq x \leq 1$ and define $\varepsilon$ via the relationship
\begin{equation*}
\left\lVert A^{1-x}B^{x}\right\rVert ^{2} = 1-\varepsilon.
\end{equation*}
This means that for the principal right singular vector $v$ with $\left\lVert v\right\rVert =1$, we
have 
\begin{equation*}
\left\langle A^{1-x}B^{x}v,A^{1-x}B^{x}v\right\rangle _{\mathbb{R}}\geq1-\varepsilon.
\end{equation*}

Applying Lemma 3, we deduce the existence of $t \in \mathbb{R}$ such that
$$
\Re F\left(1+it\right)=\Re\left\langle A^{-it}B^{1+it}v,A^{-it}B^{1+it}v\right\rangle _{\mathbb{R}}\geq1-\frac{\varepsilon}{x}
$$
Using the definition of complex powers, we have
$$
A^{-it}B^{1+it}v=\sum_{k=1}^{n}\lambda_{k}^{-it}\left\langle B^{1+it}v,a_{k}\right\rangle a_{k}
$$
allowing us to rewrite $\Re F\left(1+it\right)$ as
\begin{align*}
\Re F\left(1+it\right)&=\Re\left\langle A^{-it}B^{1+it}v,A^{-it}B^{1+it}v\right\rangle _{\mathbb{R}}  \\
& =\Re\left\langle \sum_{k=1}^{n}\lambda_{k}^{-it}\left\langle B^{1+it}v,a_{k}\right\rangle a_{k},\sum_{k=1}^{n}\lambda_{k}^{-it}\left\langle B^{1+it}v,a_{k}\right\rangle a_{k}\right\rangle _{\mathbb{R}}\\
& =\Re \sum_{k=1}^{n}\sum_{k'=1}^{n}\lambda_{k}^{-it}\lambda_{k'}^{-it}\left\langle B^{1+it}v,a_{k}\right\rangle \left\langle B^{1+it}v,a_{k'}\right\rangle\left\langle a_{k}, a_{k'}\right\rangle
_{\mathbb{R}} \\
&=  \sum_{k=1}^{n}\Re\left(\lambda_{k}^{-2it}\left\langle B^{1+it}v,a_{k}\right\rangle ^{2}\right) 
\end{align*}
This implies the existence of $t \in \mathbb{R}$ such that
$$
1 - \frac{\varepsilon}{x} \leq \sum_{k=1}^{n}\Re\left(\lambda_{k}^{-2it}\left\langle B^{1+it}v,a_{k}\right\rangle ^{2}\right) \leq \sum_{k=1}^{n}\left\vert\lambda_{k}^{-2it}\left\langle B^{1+it}v,a_{k}\right\rangle ^{2}\right\vert.
$$

We first observe that 
$$\lambda^{it}=\cos\left(\left(\log\lambda\right)t\right)+i\sin\left(\left(\log\lambda\right)t\right)=e^{i\left(\log\lambda\right)t}$$
and therefore $\lambda_k^{-2it}$ has a magnitude of 1. Therefore
$$ 1 - \frac{\varepsilon}{x} \leq \sum_{k=1}^{n} \left|\left\langle B^{1+it}v,a_{k}\right\rangle\right|^{2} $$
For any $w \in \mathbb{C}^n$, since the eigenvectors $a_k$ form a basis of $\mathbb{R}^n$,
\begin{align*}
\sum_{k=1}^{n} \left|\left\langle \Re w + i \Im w,a_{k}\right\rangle\right|^{2} &= \sum_{k=1}^{n}  \left|\left\langle \Re w ,a_{k}\right\rangle\right|^{2} +  \left|\left\langle \Im w ,a_{k}\right\rangle\right|^{2}  \\
&= \left\| \Re w\right\|^2 + \left\| \Im w\right\|^2 = \left\langle w, \overline{w} \right\rangle = \left\|w\right\|^2.
\end{align*}
Therefore, recalling that purely imaginary powers are unitary, 
$$ 1 - \frac{\varepsilon}{x} \leq \left\| B^{1+it} v\right\|^2 =  \left\| B^{} v\right\|^2.$$
Using the spectral theorem combined with the fact that the largest eigenvalue of $B$ is 1 (and simple) and the second largest eigenvalue is $\mu_2 < 1$, we have
\begin{align*}
 1 - \frac{\varepsilon}{x}  \leq  \left\| B^{} v\right\|^2 &\leq \left\langle v, b_1\right\rangle^2 + \mu_2 \| \pi_{b_1^{\perp}} v \|^2 \\
  &= \left\langle v, b_1\right\rangle^2 + \mu_2  \left(1-\left\langle v, b_1\right\rangle^2\right) \\
  &= \mu_2 + (1-\mu_2) \left\langle v, b_1\right\rangle^2 
  \end{align*}
This implies the desired inequality for $B$.
The exact same analysis can be applied to the boundary $z=0+it$, yielding the second part of the statement
\end{proof}

\subsection{Refinements}
\label{subsec:better_bound}
 The proof implies a slightly stronger statement. It is easily seen that the argument shows, for example, that
for any vector $v \in \mathbb{R}^n$ for which $\|A^{1-x}B^{x}v\| = \|A^{1-x}B^{x}\|$, we automatically have that 
$$\left\| Bv \right\|^{2} \geq 1-\frac{1-\|A^{1-x}B^{x}\|^2}{x }$$
which forces $\|Bv\|$ to be large. $\|Bv\|$ being large in combination with a spectral gap automatically forces that $v$ has 
a large inner product with the leading eigenvector. However, this is also the worst case; in practice, one would perhaps expect
that the vector $v \in \mathbb{R}^n$ for which $\|A^{1-x}B^{x}v\| = \|A^{1-x}B^{x}\|$ has a nontrivial inner product also with other
(smaller) eigenvectors of $B$ which then implies an even stronger concentration for the leading eigenvectors.
Following the proof of Theorem~\ref{theorem:iff_bound}, we derive a tighter bound by relaxing the reliance on the spectral gap $1-\mu_2^2$.
Resuming from $\|B v\|^2 \geq 1 - \varepsilon/x$ and expanding $v$ in the eigenbasis $\{b_m\}$ of $B$ (where $\mu_1=1$), we obtain:
$$
\sum_{m=2}^n (1 - \mu_m^2) |\langle v, b_m \rangle|^2 \leq \frac{\varepsilon}{x}.
$$
We normalize this inequality by the tail mass $1 - |\langle v, b_1 \rangle|^2 = \sum_{m=2}^n |\langle v, b_m \rangle|^2$. Defining the second moment of the tail spectrum as
$$
\rho_b \triangleq \frac{\sum_{m=2}^n \mu_m^2 |\langle v, b_m \rangle|^2}{\sum_{m=2}^n |\langle v, b_m \rangle|^2},
$$
the inequality simplifies to $(1 - |\langle v, b_1 \rangle|^2)(1 - \rho_b) \leq \frac{\varepsilon}{x}$. Rearranging terms yields the improved bound:
$$
|\langle v, b_1 \rangle|^2 \geq 1 - \frac{\varepsilon}{x(1 - \rho_b)}.
$$
This tightens the bound since $\rho_b \leq \mu_2^2$, with $\rho_b$ becoming smaller if the error aligns with high-frequency modes (small $\mu_m$). By symmetry, an analogous bound holds for the left singular vector using $A$.

\section{Application to Multi-Manifold Learning}
\label{section:app_to_manifold}
Multimodal manifold learning deals with the fundamental challenge of representing data from diverse sources and modalities. This task is crucial for data analysis, as it helps describe relationships between different data modalities, a core challenge, and a shared goal across many domains and applications.

\subsection{Setting}
\label{subsec:setting}

Consider three hidden manifolds ${\mathcal{M}}_1$, ${\mathcal{M}}_2$, and ${\mathcal{M}}_3$, which are observed through two observation functions
    \begin{align*}
        g &\colon {\mathcal{M}}_1\times{\mathcal{M}}_2\times{\mathcal{M}}_3 \to {\mathbb{S}}_1  \\
        h &\colon {\mathcal{M}}_1\times{\mathcal{M}}_2\times{\mathcal{M}}_3 \to {\mathbb{S}}_2,  
    \end{align*}
where $\mathbb{S}_1$ and $\mathbb{S}_2$ are subsets of (possibly different) Euclidean spaces. We can think of the triplet $(\mathcal{M}_1, \mathcal{M}_2, \mathcal{M}_3)$ as the underlying global structure and of the functions $g$ and $h$ as two different ways of extracting information (e.g., these functions may represent samples captured by two different sensors). 
Following \cite{LEDERMAN2018509,talmon2019latent}, we assume that $g$ is a smooth isometric embedding of $\mathcal{M}_1 \times \mathcal{M}_3$ into $\mathbb{S}_1$, ignoring $\mathcal{M}_2$, and $h$ is a smooth isometric embedding of $\mathcal{M}_2 \times \mathcal{M}_3$ into $\mathbb{S}_2$, ignoring $\mathcal{M}_1$. This assumption implies that $\mathcal{M}_3$ represents the common component of the observed data (which is often the desired piece of information), while $\mathcal{M}_1$ and $\mathcal{M}_2$ represent observation-specific perspectives (often associated with interferences).
The problem at hand is to obtain a representation of the common component given observations through $g$ and $h$ (see \cref{Fig:formal_setup}).

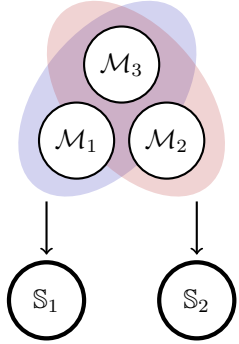
\begin{figure}[h!]
    \centering

\begin{tikzpicture}
    \fill[ rotate=50,fill={rgb:red,1;green,1;blue,5}, fill opacity=0.2] (0.3, 0.5) ellipse (1.5cm and 0.9cm);
     \fill[ rotate=-50,fill={rgb:red,5;green,1;blue,1}, fill opacity=0.2] (-0.3, 0.5) ellipse (1.5cm and 0.9cm);
    \draw[thick, fill=white] (0,1) ellipse (0.5cm and 0.5cm);
    \node at (0,1) {$\mathcal{M}_3$};
        \draw[thick, fill=white] (-0.6,0) ellipse (0.5cm and 0.5cm);
    \node at (-0.6,0) {$\mathcal{M}_1$};
        \draw[thick, fill=white] (0.6,0) ellipse (0.5cm and 0.5cm);
    \node at (0.6,0) {$\mathcal{M}_2$};

\draw [thick, ->] (-1, -0.8) -- (-1, -1.5);
\draw [thick, ->] (1, -0.8) -- (1, -1.5);
    \draw[ultra thick] (-1,-2.1) ellipse (0.5cm and 0.5cm);
    \node at (-1,-2.1) {$\mathbb{S}_1$};
    \draw[ultra thick] (1,-2.1) ellipse (0.5cm and 0.5cm);
        \node at (1,-2.1) {$\mathbb{S}_2$};
\end{tikzpicture}
    \captionsetup{width=\textwidth}
    \caption[Illustration of formal manifold setup]{A sketch of the multimanifold learning setup: three hidden manifolds and two observables, where at each observation only one manifold ($\mathcal{M}_3$) is common.}
    \label{Fig:formal_setup}
\end{figure}

Consider $n$ inaccessible samples $\{(x_i, y_i, z_i)\}_{i=1}^{n}$ from some joint distribution supported on the product of the hidden manifolds $\mathcal{M}_1 \times \mathcal{M}_2 \times \mathcal{M}_3$, which give rise to $n$ pairs of accessible data points $\{(s^{(1)}_i, s^{(2)}_i)\}_{i=1}^{n}$ such that $s^{(1)}_i = g(x_i, y_i, z_i)$ and $s^{(2)}_i = h(x_i, y_i, z_i)$.
The two sets of samples are viewed as a discretization of the respective underlying manifolds. We compute two kernels, one for each set, consisting of pairwise local affinities between the observed samples. Typically, the positive Gaussian kernel is used to measure local affinities, i.e.,
\begin{align}
    A_{i,j} &= \exp\left(-\|s_i^{(1)}-s_j^{(1)}\|^2_2/\varepsilon^{(1)}\right) \label{eq:affinity_kernels1}\\
    B_{i,j} &= \exp\left(-\|s_i^{(2)}-s_j^{(2)}\|^2_2/\varepsilon^{(2)}\right)
    \label{eq:affinity_kernels2}
\end{align}
for $i,j=1,\ldots,n$, where $\varepsilon^{(1)},\varepsilon^{(2)}>0$ are two scale parameters. After applying the conventional normalizations to the kernels (see \cite{coifman2006diffusion}), we obtain two symmetric positive definite matrices $A, B \in \text{Sym}_n^{+} \subset \mathbb{R}^{n \times n}$, where $\|A\|=\|B\|=1$.

\subsection{Algorithm}
Here, following \cite{Katz2025}, we take an approach that relies on kernel interpolation.  We apply the eigenvalue decomposition to the normalized kernels $A$ and $B$ to obtain their eigenvalues, denoted as $\lambda_k$ and $\mu_m$, respectively. 
To analyze the relationship between the two sets of measurements, we interpolate between $A$ and $B$ according to the continuous map $\gamma:[0,1]\rightarrow \mathbb{R}^{n \times n}$ given by
\begin{equation}
\label{eq:simple_geometric_mean_weighted_interpolation}
    \gamma(x) = A^{1-x} B^{x}.
\end{equation} 
For the interpolation, we use a regular grid with $M+1$ points, $x_i = \frac{i}{M}$ for $0\leq i \leq M$, yielding $M+1$ symmetric positive-definite matrices $\gamma(x_i)$. 
To each matrix $\gamma(x_i)$, we apply the singular value decomposition and obtain the top $K$ singular values $\{\sigma_{x_i}^{k}\}_{k=1}^{K}$. 
We then generate a diagram depicting the variation of the top $K$ singular values across the interpolated points by plotting them along the interpolation axis $x$, where for each $x_i$ we have $K$ singular values (see \cref{subsec:numerical_results}). For the special cases $x_i=0$ and $x_i=1$, the singular values coincide with the eigenvalues: $\sigma_{0}^{k} = \lambda_k$ and $\sigma_{1}^{k} = \mu_k$ for $k=1,\ldots,K$.
We summarize this in \cref{alg:svfd}.

\begin{algorithm}[h!]
\caption[Singular Values Flow Diagram (SVFD) Generation]{Singular Values Flow Diagram Generation}
\label{alg:svfd}
\begin{algorithmic}[l]
  \Input Two sets of aligned measurements, $\{s^{(1)}_i\}_{i=1}^{n} \subset \mathbb{S}_{1},\ \{s^{(2)}_i\}_{i=1}^{n} \subset \mathbb{S}_{2}$ 
  \Output Singular values diagram 
   \Parameters \\
   	\begin{itemize}
		\item $M$ -- The number of points on the interpolation axis 
		\item $K$ -- The number of singular values at each interpolation point
	\end{itemize}

\State Build two kernels for the two input sets of measurements:
		\begin{algsubstates}
			\State Construct two SPD affinity kernels $A$ and $B$ using \cref{eq:affinity_kernels1} and \cref{eq:affinity_kernels2}.
			\State Normalize the kernels.
		\end{algsubstates}	
		\State Consider a regular grid of $M+1$ points $\{x_i=\frac{i}{M}\}_{i=0}^{M}$ in $[0,1]$. \\
        For each $x_i$:
		\begin{algsubstates}
			\State Compute the matrix $A^{1-x_i}B^{x_i}$
			\State Apply SVD and obtain the largest $K$ singular values $\{\sigma_{x_i}^{k}\}_{k=1}^{K}$
			\State Scatter plot the logarithm of the obtained singular values as a function of $x_i$.
		\end{algsubstates}	
\end{algorithmic}
\end{algorithm}

\subsection{An example}
\label{subsec:numerical_results}
We demonstrate our method on a pair of cylindrical surfaces, denoted by \( \mathcal{C}_1 \) and \( \mathcal{C}_2 \) and illustrated in \cref{Fig:cylinders_scatter}.
We sample $n=1000$ tuples $\left\{\left(x_i, y_i, z_i\right)\right\}^{n}_{i=1}$ uniformly from a product of three hidden 1D manifolds $\mathcal{S}^1 \times \mathcal{S}^1 \times \left[0,2\pi\right]$, where $x_i \in \mathcal{S}^1, y_i \in \mathcal{S}^1, z_i \in [0,2\pi] $, and $\mathcal{S}^1$ denotes the 1D sphere. These samples are then mapped onto two cylindrical surfaces using the functions $g, h$
\begin{align*}
s^{\left(1\right)}_i = g(x_i,y_i,z_i) = \frac{1}{2\pi}  
\begin{pmatrix}
P_1\cos(x_i)  \\
P_1\sin(x_i) \\
L_1 \cdot z_i \\
\end{pmatrix}, \ \ 
s^{\left(2\right)}_i =h(x_i,y_i,z_i) = \frac{1}{2\pi}  
\begin{pmatrix}
P_2\cos(y_i)  \\
P_2\sin(y_i) \\
L_2 \cdot z_i \\
\end{pmatrix},
\end{align*}
where the parameters are set to $L_1=2, P_1=1.25, L_2=2, P_2=3$.
The mapped samples are viewed as observations on 2D cylinders embedded in \( \mathbb{R}^3 \), where the common variable \( z_i \in [0,2\pi] \) represents the height coordinate, and \( x_i \in \mathcal{S}^1 \) and \( y_i \in \mathcal{S}^1 \) are distinct and represent azimuthal angles. 
A 2D cylindrical surface with Neumann boundary conditions has a spectrum that is analytically tractable. Specifically, the eigenvalues of $\mathcal{C}_1$ and $\mathcal{C}_2$ are given respectively by the following closed-form expressions:
\[ \lambda^{\left(k_x, k_z\right)}_1 = \left(\frac{\pi k_z}{L_1}\right)^2 + \left(\frac{2\pi}{P_1} \left\lfloor \frac{k_x}{2} \right\rfloor\right)^2,\ \\
\lambda^{\left(k_y, k_z\right)}_2 = \left(\frac{\pi k_z}{L_2}\right)^2 + \left(\frac{2\pi}{P_2} \left\lfloor \frac{k_y}{2} \right\rfloor\right)^2,
\]
where $k_x, k_y, k_z = 0, 1, 2, \ldots$ are indices. We see in these expressions that the ``degree of commonality'' is determined by the ratio between the shared height, $L_1$ or $L_2$, and the distinct perimeter, $P_1$ or $P_2$, respectively: a small ratio pushes common eigenvalues deeper in the spectrum, whereas a large ratio does so for distinct ones.
We apply Algorithm \ref{alg:svfd} with $M=51$ to the two sets of samples on the two cylinders. 
In \cref{Fig:cylinders_diagram_vec_2_2}, we plot the resulting singular values diagram  (gray). On the boundaries, at \( x=0 \) and \( x=1 \), using the following relation~\cite[equation (7)]{dsilva2015parsimonious}:
\begin{equation}
    \tilde{\lambda} = \exp\left(-\frac{\varepsilon^2}{4}\lambda\right),
    \label{eq:exponent_eigenvalues_svfd_relation}
\end{equation}
we overlay three common analytical eigenvalues of the two cylinders (setting $k_x=k_y=0$) on the empirical eigenvalues of \( \mathcal{C}_1 \) and \( \mathcal{C}_2 \) and mark them by blue squares. Dashed lines show the log-linear interpolation between corresponding analytical eigenvalues (with the same $k_z$ index). We see that the resulting empirical singular values at any interpolated point \( x \in (0,1) \) nearly coincide with the log-linear interpolation between the analytical spectrum.
In \cref{Fig:distinct_cylinders_diagram_vec_4}, we show the same SVFD as in \cref{Fig:cylinders_diagram_vec_2_2}, but now highlight the empirical singular values corresponding to two non-common spectral components. Specifically in \cref{Fig:distinct_cylinders_diagram_vec_4}, we examine the fourth-largest eigenvector of $\mathcal{C}_1$, which corresponds to azimuthal oscillations. 
 The SVFD presents common and non-common spectral components differently: curves associated with eigenpairs that share the common height variable are approximately straight curves that closely follow the dashed interpolations, while eigenpairs dominated by the distinct azimuthal variables give rise to curved trajectories.

\subsection{Concluding remarks} 
The proposed interpolation \( \gamma(x) = A^{1-x} B^{x} \) in \eqref{eq:simple_geometric_mean_weighted_interpolation} enables a separation between common and non-common spectral components. It is efficient and mathematically tractable, and we show both theoretically and empirically that it conveys not only dichotomous information, but also the degree of commonality of the components.
However, the considered interpolation is not unique and raises several questions. For example, does the order of the matrices in the product affect the result? 
For instance, symmetric interpolations such as \( B^{x/2} A^{1-x} B^{x/2} \), \( A^{1-x} B^{2x} A^{1-x} \), or \( B^{x} A^{2(1-x)} B^{x} \) can be considered. In \cite{Katz2025}, another symmetric interpolation scheme based on the geodesic between two symmetric positive-definite matrices under the affine-invariant metric \cite{pennec2006riemannian,Bhatia} was presented. Similarly, one could consider geodesics, or other trajectories on the symmetric positive-definite manifold, induced by different Riemannian metrics. 
We have established a theoretical framework and provided tools for such an approach to multimodal manifold learning via kernel interpolation; we believe that the theorems presented here can serve as a blueprint for what is possible more generally.

\section*{Acknowledgments}
This work was funded by the European Union's Horizon 2020 research and innovation programme under Grant 802735-ERC-DIFFOP.

\bibliographystyle{siamplain}
\bibliography{references}
\end{document}

%% file: main_shared.tex
\usepackage{lipsum}
\usepackage{amsfonts}
\usepackage{graphicx}
\usepackage{epstopdf}
\usepackage{caption}
\usepackage{thmtools}
\usepackage{subfig}
\usepackage{tikz}
\usepackage{amssymb}

\usepackage{algorithmicx}

\usepackage[noend]{algpseudocode}

\newcounter{algsubstate}

\makeatother
\newenvironment{algsubstates}
{\setcounter{algsubstate}{0}%
	\renewcommand{\State}{%
		\refstepcounter{algsubstate}
		\Statex {\footnotesize\alph{algsubstate}:}\space}}
{}

\ifpdf
  \DeclareGraphicsExtensions{.eps,.pdf,.png,.jpg}
\else
  \DeclareGraphicsExtensions{.eps}
\fi

\newsiamremark{remark}{Remark}
\newsiamremark{hypothesis}{Hypothesis}
\crefname{hypothesis}{Hypothesis}{Hypotheses}
\newsiamthm{claim}{Claim}
\newsiamremark{fact}{Fact}
\crefname{fact}{Fact}{Facts}

\newcommand{\RR}{\mathbb{R}}
\newcommand{\CC}{\mathbb{C}}

\algnewcommand\algorithmicinput{\hspace*{\algorithmicindent}\textbf{Input:}}
\algnewcommand\Input{\item[\algorithmicinput]}

\algnewcommand\algorithmicoutput{\hspace*{\algorithmicindent}\textbf{Output:}}
\algnewcommand\Output{\item[\algorithmicoutput]}

\algnewcommand\algorithmicparams{\hspace*{\algorithmicindent}\textbf{Parameters:}}
\algnewcommand\Parameters{\item[\algorithmicparams]}

\headers{Complex Interpolation of Matrices}{A. Arbel, S. Steinerberger, and R. Talmon}

\title{Complex Interpolation of Matrices with an application to Multi-Manifold Learning
}

\author{Adi Arbel\thanks{Viterbi Faculty of Electrical and Computer Engineering, Technion -- Israel Institute of Technology,
Haifa, Israel (\email{adi.arbel@campus.technion.ac.il}).}
\and Stefan Steinerberger\thanks{Department of Mathematics and Department of Applied Mathematics, University of Washington, Seattle, WA 98195, USA 
  (\email{steinerb@uw.edu}).}
\and Ronen Talmon\thanks{Viterbi Faculty of Electrical and Computer Engineering, Technion -- Israel Institute of Technology,
Haifa, Israel (\email{ronen@ee.technion.ac.il}).}}
\usepackage{amsopn}
